\documentclass[sigconf]{acmart}
\AtBeginDocument{%
  }

\usepackage{booktabs}
\usepackage{multirow}

\usepackage[normalem]{ulem}
\useunder{\uline}{\ul}{}
\usepackage[ruled,noline]{algorithm2e}
\SetKwInput{KwModel}{Model}
\SetKwInput{KwRequire}{Require}
\SetKwInput{KwLoss}{Loss}

\usepackage{amsthm}

\newtheorem{theorem}{Theorem}

\newtheorem{assumption}{Assumption}[section]
\newtheorem{definition}{Definition}[section]

\usepackage{svg}
\usepackage{bbm}
\usepackage{graphicx}
\usepackage{adjustbox}

\newcommand{\indep}{\perp \!\!\! \perp}

\usepackage{amsmath}

\usepackage{natbib}

\setcopyright{cc}
\copyrightyear{2026}
\acmYear{2026}
\setcopyright{cc}
\setcctype{by}
\acmConference[CIKM '26]{Proceedings of the 35th ACM International Conference on Information and Knowledge Management}{November 07--11, 2026}{Rome, Italy}
\acmBooktitle{Proceedings of the 35th ACM International Conference on Information and Knowledge Management (CIKM '26), November 07--11, 2026, Rome, Italy}
\acmDOI{10.1145/3799682.3840587}
\acmISBN{979-8-4007-2539-5/2026/11}

\begin{document}

\title{A Deep Learning Framework for Quantile Demand Forecasting in E-commerce with Doubly Censored Data}

\author{Ziwen Ye}
\authornote{This work was done while the author was at Meituan.}
\email{guziyzw@gmail.com}
\orcid{0009-0005-6908-944X}
\affiliation{%
  \institution{Meituan}
  \state{Beijing}
  \country{China}
}

\author{Jiawen Wei}
\authornote{corresponding Author.}
\email{weijiawen@meituan.com}
\affiliation{%
  \institution{Meituan}
  \state{Beijing}
  \country{China}
}

\author{Teng Song}
\email{songteng02@meituan.com}
\affiliation{%
  \institution{Meituan}
  \state{Beijing}
  \country{China}
}
  
\author{Chen Chen}
\email{chenchen11@meituan.com}
\affiliation{%
  \institution{Meituan}
  \state{Beijing}
  \country{China}
  }
\author{Guangrui Ma}
\email{magr@connect.ust.hk}
\affiliation{%
  \institution{Meituan}
  \state{Beijing}
  \country{China}
}

\renewcommand{\shortauthors}{Ziwen Ye, Jiawen Wei, Teng Song, Chen Chen, and Guangrui Ma}

\begin{abstract}
In recent years, a growing number of online retail companies have embraced deep learning models for demand forecasting, leveraging their scalability in processing large-scale datasets. Accurate demand forecasts enable improved planning and allow operational flexibility to adjust sales accordingly. Moreover, quantile estimation has gained widespread application as it offers valuable uncertainty quantification, thereby enhancing the robustness of decision-making processes. However, limited research has addressed the effective quantification of uncertainty when target labels are ambiguous. In retail, actual sales can be capped by stock and boosted by temporary promotions, whereas the target is underlying demand. Such scenarios are characterized by double censoring and no deep-learning pipeline yet exists for such data. In this paper, we introduce a straightforward framework for learning quantiles from doubly-censored observations and propose a tailored evaluation metric. Our approach is supported by theoretical guarantees, and we empirically demonstrate that our metric correlates well with the oracle calibration measure. Experimental results show that our estimator surpasses existing baselines. Finally, inventory simulations and real-world applications confirm the practical utility of our method in retail operations.
\end{abstract}




\maketitle

\section{Introduction}

In supply chain management, quantile regression is highly valued for its distribution-free nature and ability to estimate uncertainty. Retail giants like JD and Amazon utilize deep learning-based quantile estimation for inventory optimization\citep{qi2023practical,park2022learning,petropoulos2022forecasting}. However, when sales data are censored, such as due to stockouts or promotions, quantile regression estimates can become biased. As illustrated in \textbf{Figure \ref{fig:Figure 1}}, right-censoring happens with perishable items, where low stock causes frequent stockouts and conceals true demand. On the flip side, left-censoring occurs with promotional products, where temporary promotions lead to observed sales surpassing actual demand. Daily necessities typically have sufficient stock and rarely go on promotion, so their sales data usually reflects the actual demand. When right and left censoring coexist in one dataset, it's called doubly censoring.

\begin{figure}[!htbp]
  \centering
\includegraphics[width=\columnwidth, scale = 0.5, keepaspectratio]{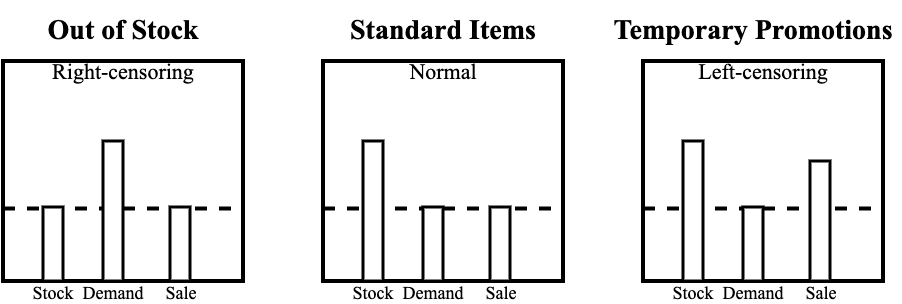}
  \caption{Censored data example.}
  \label{fig:Figure 1}
\end{figure}

Using or discarding doubly censored data both distort forecasts: the former fits the model on biased labels, the latter discards usable information\citep{zhong2021deep}. In practice, retailers often revise sales using ad hoc business rules, yet without true demand they cannot verify the accuracy or feed the revised data back into the forecasting pipeline. Consequently, demand forecasting still lacks an end-to-end training method for doubly censored data and an appropriate metric to evaluate models under this scenario. Survival analysis offers a direct parallel: only a lower bound on lifetime is observed, producing right-censored responses. Re-weighting techniques developed for this setting motivate our unified framework, which trains any deep network end-to-end with the Doubly Portnoy Estimator (DPE). Likewise, while quantile forecasts are usually judged by calibration and sharpness, no calibration test exists for doubly censored data. Borrowing again from survival methodology, we propose Extended D-Calibration to assess the calibration of competing models under double censoring.

In conclusion, this paper presents a unified framework for training and evaluating deep models under doubly censored data. Although motivated by e-commerce demand forecasting, the framework is domain-agnostic and can be extended to any setting with similar censoring patterns. Our contributions are as follows.
\begin{itemize}
    \item We introduce an end-to-end algorithm that trains deep models with the DPE, and provide the first formal proof of its properness under doubly censored data, closing a long-standing theoretical gap in quantile estimation with incomplete observations.
    \item To evaluate models trained under such censoring, we propose Extended D-Calibration, a new metric for doubly censored data, and prove that it precisely gauges the gap between modelled and true probabilities under mild assumptions.
    \item We first validate the proposed metric and algorithm through simulations, then confirm their effectiveness on two large-scale real-world datasets. Inventory simulations further demonstrate their superiority in stock control, and online A/B tests verify their practical impact.
\end{itemize}

\section{Related Work}

\subsection{Probabilistic Demand Forecasting}

Probabilistic demand forecasting plays a crucial role in inventory management across various sectors. A widely studied problem in this domain is the newsvendor problem, which identifies a specific quantile of the demand distribution as the optimal inventory level\citep{ban2019big,beutel2012safety}. Traditional approaches for estimating demand quantiles typically rely on deterministic forecasts and the historical residuals of predicted values\citep{levi2015data}. With the rise of deep learning, neural network-based demand forecasting has become increasingly prevalent. For example, to capture uncertainty, DeepAR employs a Recurrent Neural Network (RNN) to generate hidden states and estimates distribution parameters through maximum likelihood estimation\citep{salinas2020deepar}. This  approach requires assumptions about the underlying data distribution.

In contrast, quantile regression offers a non-parametric approach for probabilistic forecasting that does not require distributional assumptions\citep{koenker1978regression}. Building on this foundation, Wen, Lim, and other researchers\citep{wen2017multi,lim2021temporal} have designed Seq2Seq models that directly estimate quantiles using quantile loss as the loss function. However, these models typically provide predictions for a fixed set of quantiles. To enable arbitrary quantile forecasting, a spline-based neural network have been proposed\citep{gasthaus2019probabilistic,park2022learning,taieb2022learning,brando2022deep}. These networks leverage the flexibility of splines to interpolate between quantiles, thereby allowing for more granular probabilistic forecasts.

In the theoretical analysis of probabilistic forecasting, the theory of scoring rules serves as a fundamental framework. A loss function that qualifies as a proper scoring rule ensures that model training with this function can elicit the true underlying distribution\citep{gneiting2007strictly}. To evaluate the effectiveness of different proper scoring rules, recent studies have widely adopted calibration and sharpness as key evaluation criteria\citep{chung2021beyond,taieb2022learning,feldman2021improving}. Our study also employs these metrics to evaluate model performance.


\subsection{Censored Quantile Regression}
As highlighted in the introduction, censored data remain largely overlooked in demand forecasting. To exploit such data, we draw upon survival analysis, a field with an extensive literature on censoring (see \citep{wang2019machine,peng2021quantile} for a comprehensive survey). 

Deep learning has recently been combined with survival analysis under right censoring \citep{zhong2021deep,katzman2018deepsurv}.
Some studies use censored data during training \citep{avati2020countdown,goldstein2020x}. In the statistical literature, doubly censored data have been studied via quantile regression \citep{lin2012quantile,ji2012quantile,peng2021quantile}, but few works link censored quantile regression to deep models.
The \textsf{q-rnn} package trains networks with left censoring \citep{Cannon2018}.
Pearce et al.\ extend Portnoy's Estimator to train deep models via an iterative re-weighting algorithm \citep{pearce2022censored}.
Yanagisawa shows that Portnoy's Estimator is proper, so the true distribution is recovered by optimization \citep{yanagisawa2023proper}.


On the evaluation side, the C-index \citep{harrell1982evaluating} quantifies median discrimination only, whereas D-Calibration \citep{haider2020effective} assesses calibration under right censoring. Neither extends to doubly censored data. We close this gap by introducing Extended D-Calibration, the first metric to evaluate calibration in the presence of both left and right censoring.

\section{Preliminaries}
\subsection{Notation}
We consider the task of using sequential features $x \in \mathcal{X}$ to learn the quantile function of demand $y \in \mathcal{Y}$ with a censoring indicator $\Delta$. Given a dataset $\mathcal{D} = \{(x_1, y_1^{*}, \Delta_1), ..., (x_N, y_N^{*}, \Delta_N) \}$, $y^{*}$ denotes the observed demand and $y$ denotes the true demand. When no censoring occurs, $y^* = y$ and we set $\Delta = 0$. Let $R$ and $L$ be the right-censored demand and left-censored demand, respectively. In the right-censoring scenarios, we can only observe $R$, where we denote $\Delta$ as 1. Similarly, $\Delta$ is denoted as 2 for left-censoring scenarios, where $L$ is observed. When no censoring occurs, the $\tau$-quantile of $y$ conditional on $x$ is obtained by:
\begin{align*}
y_\tau = F^{-1}(\tau|x) = \underset{z \in R}{\mathrm{argmin}} \mathop{\mathbb{E}}_{Y \sim F}[\rho_\tau(y, z)]
\end{align*}
where $\rho_\tau(y, z) = (\tau - \mathbbm{1}[y < z])(y - z)$ is the quantile loss, $\tau \in (0, 1)$ is the quantile level and $\mathbbm{1}$ is the indicator function\citep{koenker1978regression}. Quantile loss has been shown to be a proper scoring rule\citep{gneiting2007strictly}. The formal definition of the proper scoring rule under censored data is given as follows\citep{yanagisawa2023proper},
\begin{definition}
    A scoring rule $ S(\hat{F}, (y, \Delta))$ is proper if 
    \begin{align*}
        \mathop{\mathbb{E}}_{y \sim Y}[S(\hat{F}, (y, \Delta))] \ge \mathop{\mathbb{E}}_{y \sim Y}[S(F, (y, \Delta))].
    \end{align*}
\end{definition}
Where $F$ is the cumulative distribution function (CDF) and $\hat{F}$ is the estimation of $F$. In probabilistic forecasting, training with a proper scoring rule is standard: if the loss is not proper, the fitted distribution $\hat{F}$ may deviate arbitrarily from the true $F$.

Unless stated otherwise, we adopt the standard survival-analysis assumption that censoring is independent of the true demand conditional on the features. This assumption can be expressed as follows:
\begin{assumption}\label{asu:asu1}
    [Independent censoring] $(L,R) \indep y|x$
\end{assumption}

While unobserved confounders may violate this in practice, 
our covariates (e.g., price, promotions) absorb the main 
sources of endogeneity.

\subsection{Evaluation}

Calibration and sharpness are the two canonical criteria for evaluating probabilistic forecasts. For calibration, we first examine the uncensored case. To check whether the model is calibrated on an interval $I = [\tau_k, \tau_{k+1}]$, we simply count how many observations fall into that interval: $N_k = \sum\limits_{i=1}^{N}\mathbbm{1}(\tau_{iy} \in I)$, where $\tau_{iy} = \hat{F}(y_i|x_i)$ is the cumulative probability of the $i$-th observation. If $N_k / N$ equals $\tau_{k+1} - \tau_k$, the model is calibrated on $I$.  The overall calibration error of the CDF is then measured by
\begin{align*}
\text{Calibration}= \sum_{k=1}^{K-1}\bigl(\tau_{k+1}-\tau_k-\tfrac{N_k}{N}\bigr)^2.
\end{align*}

Under right-censoring, we cannot determine whether a censored observation truly falls inside $I$. Haider et al.\citep{haider2020effective} therefore re-weights each right-censored observation by its distance to the interval, using the conditional probability mass:
\begin{align*}
& N^{R}_k = \sum\limits_{i=1}^{N} \{\underbrace{\mathbbm{1}(\Delta_i =0)\mathbbm{1}(\tau_{iy} \in I)}_{\textit{uncensored observation}} \\
& + \underbrace{\mathbbm{1}(\Delta_i = 1)[\dfrac{\tau_{k+1} - \tau_{iR}}{1 - \tau_{iR}}\mathbbm{1}(\tau_{iR} \in I) + \dfrac{\tau_{k+1}- \tau_{k}}{1 - \tau_{iR}}\mathbbm{1}(\tau_{iR} < \tau_k)]}_{\textit{right-censored observation}}\} 
\end{align*}

and D-Calibration is then
\begin{align*}
\text{D-Calibration}= \sum_{k=1}^{K-1}\bigl(\tau_{k+1}-\tau_k-\tfrac{N^{R}_k}{N}\bigr)^2.
\end{align*}

Yet calibration alone does not suffice to judge forecast quality \citep{gneiting2007strictly}. When two models exhibit equal calibration, the one with the more concentrated predictive distribution is preferred. Sharpness, a property of the forecasts alone, is measured by the width of the central predictive interval. For quantiles $\tau_1$ and $\tau_2$ with $\tau_1 < \tau_2$, sharpness is defined as  
\begin{align*}
\delta = \hat{F}^{-1}(\tau_2) - \hat{F}^{-1}(\tau_1).
\end{align*}
where $\hat{F}$ denotes the estimated CDF.

\section{Methodology}
\subsection{Quantile Modeling}
In simulations we use an MLP encoder to map features into embedding representations $\mathbf{E}$. For real-world data, demand patterns are more intricate. We therefore adopt a Transformer encoder that processes historical sales, prices and auxiliary variables. It is worth noting that alternative encoders such as Temporal Fusion Transformer\citep{lim2021temporal}, PatchTST\citep{nie2022time} and TSMixer\citep{chen2023tsmixer} can replace the Transformer as needed. To enable arbitrary-quantile estimation, we model the inverse CDF as a piecewise linear function. The unit interval $[0,1)$ is split into $K$ equal segments by $K+1$ knots $\{(\tau_k, \hat{y}_k)\}_{k=1}^{K+1}$. A neural network takes $\mathbf{E}$ as input and outputs a $K$-dimensional vector $\theta$. After a softplus activation, $\theta$ is transformed into the quantile values $\{\hat{y}_{k}\}_{k=1}^{K}$. For any value $y$, we identify the interval $I_q=[\hat{y}_k,\hat{y}_{k+1}]$ such that $y \in I_q$ and compute its quantile as
\begin{align*}
\hat{\tau}_y &= (1 - w_y)\tau_k + w_y\tau_{k+1}, \\
w_y &= \frac{y - \hat{y}_{k}}{\hat{y}_{k+1} - \hat{y}_{k}}.
\end{align*}
We name this approach \textit{q\_search} and use it throughout the remainder of the paper.

\subsection{Training}

We employ DPE to estimate the quantile function under doubly censored data. For a given quantile level $\tau \in (0,1)$, the DPE loss for predicted quantile $y_\tau$ is defined as:
\begin{align*}
\begin{split}
&S_{DPE}(y^{*}, {y}_\tau, w_l, w_r, \tau, \Delta)  = \underbrace{ \mathbbm{1}(\Delta = 0) \rho_\tau({y}_\tau, y)}_{\textit{Loss for the uncensored scenario}} \\
        &  +\underbrace{ \mathbbm{1}(\Delta = 1)[w_r\rho_\tau({y}_\tau, R) + (1 - w_r)\rho_\tau({y}_\tau, z_{\infty})]}_{\textit{Loss for the right censoring scenario}} \\
        &  +\underbrace{ \mathbbm{1}(\Delta = 2)[w_l\rho_\tau({y}_\tau, L) + (1 - w_l)\rho_\tau({y}_\tau, z_{-\infty})]}_{\textit{Loss for the left censoring scenario}}
\end{split}
\end{align*}
where $z_\infty$ and $z_{-\infty}$ are sufficiently large positive and negative constants, respectively, and $w_l, w_r \in [0,1]$ are redistribution weights.

To set the weights, we examine the indicator function $\mathbbm{1}(y < y_\tau)$ in the quantile loss under censoring. For left-censoring, the observed value $L$ serves as an upper bound on the true demand ($y \leq L$). Thus, if $y_\tau > L$, we have $y \leq L < y_\tau$ with certainty, implying $\mathbbm{1}(y < y_\tau) = 1$ and setting $w_l = 1$. Conversely, if $y_\tau < L$, the indicator is ambiguous because the true $y$ may lie on either side of $y_\tau$. Following Efron's \cite{wang2009locally} redistribution principle, we impute the conditional expectation:
\begin{align*}
w_l = \mathbbm{E}[\mathbbm{1}(y > y_\tau) \mid y < L] = \frac{P(y_\tau < y < L)}{P(y < L)} = \frac{\tau_L - \tau}{\tau_L},
\end{align*}
where $\tau = F(y_\tau)$ and $\tau_L = F(L)$ are the cumulative probabilities under the true demand CDF $F$. The complementary weight $(1-w_l)$ is redistributed to $z_{-\infty}$, representing mass below the left-censoring threshold. The right-censoring weight $w_r$ is derived analogously:
\begin{align*}
w_r = \mathbb{E}[\mathbbm{1}(y < y_\tau) \mid y > R] = \frac{P(R < y < y_\tau)}{P(y > R)} = \frac{\tau - \tau_R}{1 - \tau_R}.
\end{align*}

Under the independent censoring assumption, and with $w_l, w_r$ defined with respect to the true CDF $F$, DPE induces a proper scoring rule for $\tau$-quantile estimation. This result, formalized in Theorem~\ref{thm:dpe-proper}, is to our knowledge the first properness guarantee for a quantile loss accommodating both left and right censoring.

\begin{theorem}
\label{thm:dpe-proper}
Under the independent censoring assumption, if $w_l$ and $w_r$ are correct, Doubly Portnoy's Estimator is a proper scoring rule.
\end{theorem}
\begin{proof}
We give a proof in Appendix A.
\end{proof}

Based on the weighting strategy described above, Theorem~\ref{thm:dpe-proper} establishes that DPE is a proper scoring rule for the target quantiles when the redistribution weights $w_l$ and $w_r$ are specified with respect to the true demand distribution. In practice, however, the cumulative probabilities $\tau_L = F(L)$ and $\tau_R = F(R)$ at the censoring thresholds are unknown and must be estimated from data.

\begin{algorithm}[!htbp]
\caption{Iterative Re-weighting for DPE Training}
\label{alg:one}
\KwData{Batch data $B = \{(x_j, y_j^*, \Delta_j)\}_{j=1}^b$}
\KwModel{Deep quantile model $\mathcal{M}_\theta$ with $K$ knots $\{\tau_i\}_{i=1}^K$}
\KwRequire{Target quantile $\tau$, learning rate $\alpha$, convergence tolerance $\epsilon$}
\While{$\max_j |w_{r,j}^{(t)} - w_{r,j}^{(t-1)}| > \epsilon$ \textbf{or} $\max_j |w_{l,j}^{(t)} - w_{l,j}^{(t-1)}| > \epsilon$}{
    \tcp{E-step: estimate redistribution weights (stop-gradient)}
    \For{$j = 1, \dots, b$}{
        $\{\hat{y}_{j,\tau_i}\}_{i=1}^K \leftarrow \mathcal{M}_{\theta^{(t-1)}}(x_j)$\;
        $\{\tilde{y}_{j,\tau_i}\}_{i=1}^K \leftarrow \text{SG}(\{\hat{y}_{j,\tau_i}\}_{i=1}^K)$\;
        \uIf{$\Delta_j = 0$}{
            $w_{r,j}^{(t)} \leftarrow 1$; $w_{l,j}^{(t)} \leftarrow 1$\;
        }
        \uElseIf{$\Delta_j = 1$}{
            $\hat{\tau}_R \leftarrow q\_search(\{\tilde{y}_{j,\tau_i}\}_{i=1}^K, R_j)$\;
            \uIf{$\hat{\tau}_R \leq \tau$}{
                $w_{r,j}^{(t)} \leftarrow (\tau - \hat{\tau}_R) / (1 - \hat{\tau}_R)$\;
            }
            \Else{
                $w_{r,j}^{(t)} \leftarrow 1$\;
            }
            $w_{l,j}^{(t)} \leftarrow 1$\;
        }
        \uElseIf{$\Delta_j = 2$}{
            $\hat{\tau}_L \leftarrow q\_search(\{\tilde{y}_{j,\tau_i}\}_{i=1}^K, L_j)$\;
            \uIf{$\tau \leq \hat{\tau}_L$}{
                $w_{l,j}^{(t)} \leftarrow (\hat{\tau}_L - \tau) / \hat{\tau}_L$\;
            }
            \Else{
                $w_{l,j}^{(t)} \leftarrow 1$\;
            }
            $w_{r,j}^{(t)} \leftarrow 1$\;
        }
    }
    \tcp{M-step: partial maximization via gradient descent}
    $L(\theta) \leftarrow \sum_{j=1}^b S_{DPE}(y_j^*, \hat{y}_{j,\tau}, w_{l,j}^{(t)}, w_{r,j}^{(t)}, \tau, \Delta_j)$\;
    $\theta^{(t)} \leftarrow \text{OptimizerUpdate}(\theta^{(t-1)}, \nabla_\theta L, \alpha)$\;
}
\end{algorithm}

To implement DPE in practice, we extend the iterative re-weighting (IR) procedure of \citet{pearce2022censored} from the right-censored survival setting to the doubly censored demand forecasting scenario. We extend this procedure to doubly censored demand forecasting and adapt it for deep neural networks, as summarized in Algorithm~\ref{alg:one}.

\textbf{Theoretical justification.} 
We now show that Algorithm~\ref{alg:one} admits a principled probabilistic interpretation, extending the framework of \citet{pearce2022censored} to the doubly censored setting. Following Pearce et al., we assume that the likelihood for each uncensored datapoint at quantile level $\tau$ follows an asymmetric Laplace distribution (ALD) with location $\hat{y}_{j,\tau}$ and scale $\lambda = \sqrt{\tau(1-\tau)}$. Under this assumption, the log-density corresponds to the negative check loss up to an additive constant \citep[Eq.~16]{pearce2022censored}:
\begin{align*}
\log f(y_j; \hat{y}_{j,\tau}, \lambda) = -\rho_\tau(y_j, \hat{y}_{j,\tau}) + \text{const}.
\end{align*}

For censored observations, we define the weighted likelihood by redistributing probability mass according to the current weight estimates. For a right-censored sample $j$ with censoring time $R_j$, the contribution is:
\begin{align*}
f_{\text{weighted}}(y_j; \hat{y}_{j,\tau}, w_{r,j}) = f(R_j; \hat{y}_{j,\tau})^{w_{r,j}} \cdot f(z_\infty; \hat{y}_{j,\tau})^{1-w_{r,j}},
\end{align*}
where $z_\infty$ is a sufficiently large constant representing mass beyond the censoring threshold. The left-censored case is defined analogously. Combining all samples, the joint weighted likelihood is:
\begin{align*}
p(\mathbf{y}|\mathbf{X}, \theta, \mathbf{w}_l, \mathbf{w}_r, \tau) = &\prod_{i: \Delta_i=0} f(y_i; \hat{y}_{i,\tau}) \\
&\times \prod_{j: \Delta_j=1} f_{\text{weighted}}(y_j; \hat{y}_{j,\tau}, w_{r,j}) \\
&\times \prod_{k: \Delta_k=2} f_{\text{weighted}}(y_k; \hat{y}_{k,\tau}, w_{l,k}).
\end{align*}

Taking the negative log-likelihood and substituting the ALD density recovers exactly the DPE loss $S_{DPE}$ (up to an additive constant). Thus, minimizing the DPE loss is equivalent to maximizing the weighted ALD likelihood. This reveals that Algorithm~\ref{alg:one} can be interpreted as a \textbf{generalized EM algorithm} \citep{pearce2022censored}: the E-step computes the redistribution weights by evaluating the conditional expectations under the current model, and the M-step performs a partial maximization of the expected complete-data log-likelihood via gradient descent. This probabilistic foundation elevates the iterative re-weighting strategy above heuristic reasoning, ensuring that each iteration yields a well-defined statistical objective.
Although weights are approximated from the current model, the 
generalized EM procedure converges to a stationary point 
under mild conditions~\cite{wu1983convergence}.


\subsection{Evaluation}

\begin{figure}[!htbp]
  \centering
  \includegraphics[width=\linewidth]{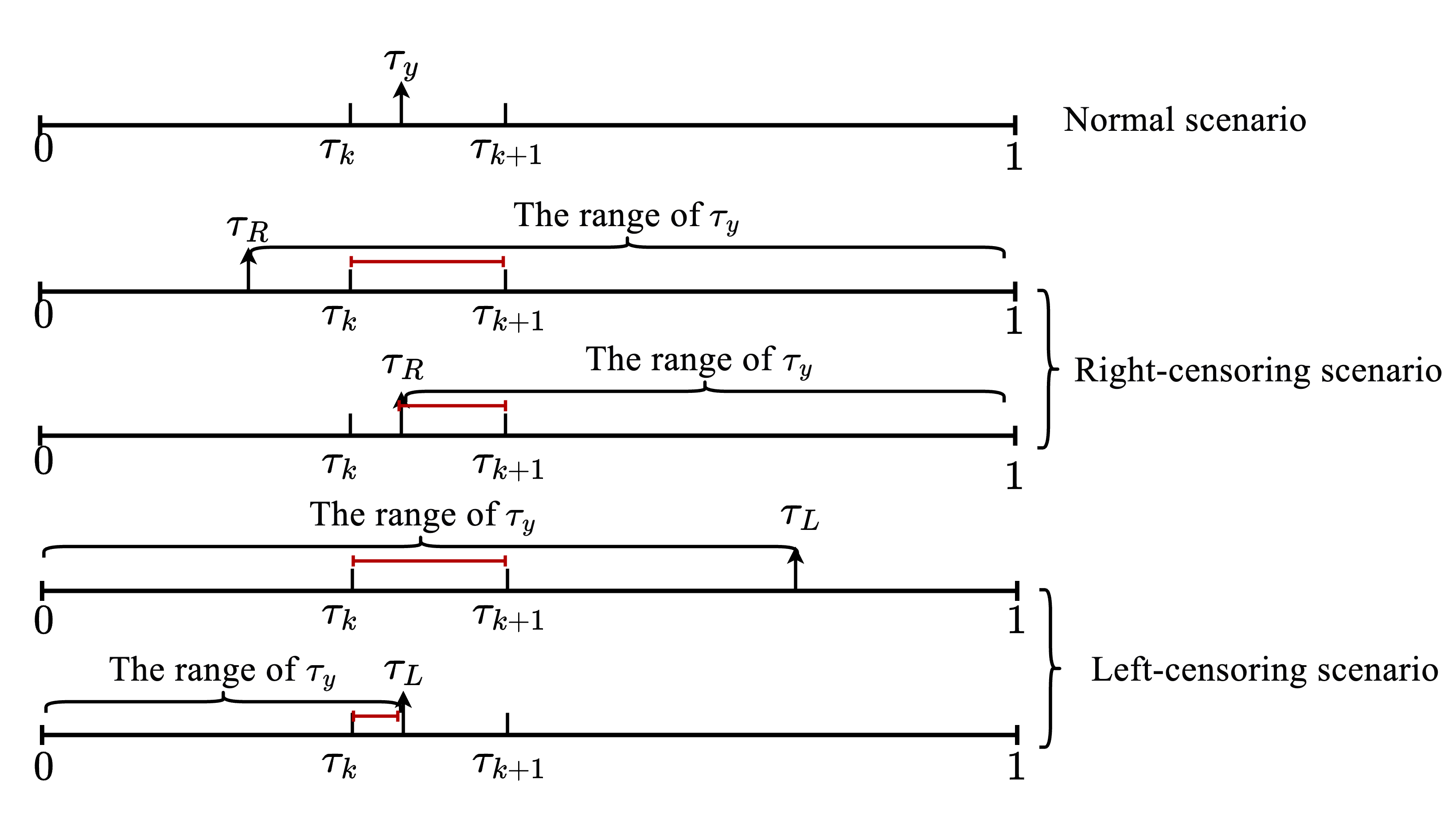}
  \caption{The illustration of re-weighting strategy.}
  \label{fig:Figure 3}
\end{figure}
Having established a training objective for doubly censored data, we now turn to evaluation. Standard calibration metrics assume fully observed outcomes and thus fail under censoring. We therefore develop Extended D-Calibration, which generalizes the re-weighting principle of DPE to assessment.

For uncensored observations, membership of $\tau_y$ in $I=[\tau_k,\tau_{k+1}]$ is certain, so calibration is computed directly.  
When data are doubly censored, ambiguity arises as soon as $\tau_L$ or $\tau_R$ lies inside $I$, because the true $\tau_y$ may still fall outside this interval. As illustrated in \textbf{Figure \ref{fig:Figure 3}}, when $\tau_R$ falls within the interval $I$, the true $\tau_y$ may occur at any point within $(\tau_R, 1)$. Thus, we design a weighting strategy for different observations. Specifically, we use the probability that $\tau_y$ falls within $I$ as their weights and incorporate these weights into the calibration calculation. Hence, the calculation of $N^{D}_k$ under doubly censored data is given by,

\begin{adjustbox}{max width=\linewidth}
$\displaystyle
\begin{aligned}
N^{D}_k = \sum_{i=1}^{N}\Bigl\{
&\underbrace{\mathbbm{1}(\Delta_i =0)\mathbbm{1}(\tau_{iy} \in I)}_{\text{uncensored}} \\
&+ \underbrace{\mathbbm{1}(\Delta_i = 1)\Bigl[
     \frac{\tau_{k+1} - \tau_{iR}}{1 - \tau_{iR}}\mathbbm{1}(\tau_{iR} \in I)
   + \frac{\tau_{k+1}- \tau_{k}}{1 - \tau_{iR}}\mathbbm{1}(\tau_{iR} < \tau_k)
   \Bigr]}_{\text{right-censored}} \\
&+ \underbrace{\mathbbm{1}(\Delta_i = 2)\Bigl[
     \frac{\tau_{iL} - \tau_k }{\tau_{iL}}\mathbbm{1}(\tau_{iL} \in I)
   + \frac{\tau_{k+1}- \tau_{k}}{\tau_{iL}}\mathbbm{1}(\tau_{iL} > \tau_{k+1})
   \Bigr]}_{\text{left-censored}}
\Bigr\}
\end{aligned}$
\end{adjustbox}
\vspace{0.5em}

Assume the unit interval $[0,1]$ is divided into $K$ non-overlapping sub-intervals. To assess the performance across these intervals, we propose a metric, Extended D-Calibration, defined as:
\begin{align*}
    \text{Extended D-Calibration} = \sum_{k = 1}^{K-1}(\tau_{k+1} - \tau_k - \dfrac{N^{D}_k}{N})^2
\end{align*}
Moreover, for the proposed $N^{D}_k$, we prove that the expectation value of $N^{D}_k/N$ equals $\tau_{k+1} - \tau_k$ under some conditions, thereby providing a theoretical guarantee for this weighting scheme. This is formalized in the following theorem:
\begin{theorem}
\label{thm:exp-dcal}
Under the independent censoring assumption, if the cumulative distribution function is strictly monotonically increasing, the expected proportion of doubly censored data in each bin equals the bin width, i.e., $\mathop{\mathbb{E}}[\dfrac{N^{D}_k}{N}] = \tau_{k+1} - \tau_{k} $
\end{theorem}
\begin{proof}
We give a proof in Appendix B.
\end{proof}

\section{Experiments}
We conduct extensive experiments to evaluate the proposed metric and algorithm on both simulated and real-world retail datasets.\footnote{Code and simulation scripts are included in the \href{https://github.com/ZiwenYeee/Framework4DoublyCensoredData}{\color{blue}{link}}.} Our three primary research objectives are:
\begin{itemize}
    \item How accurately does the proposed metric assess model calibration under doubly censored data compared with existing measures?
    \item Does the proposed method outperform alternative approaches in calibration and sharpness in both simulated and real-world scenarios?
    \item Does the method translate into operational gains, that is, superior inventory management benefits?
\end{itemize}
These three questions address the discriminative power and calibration accuracy of our metric, the predictive efficacy of our method, and its practical utility. Because no ground-truth calibration measure is available on real data, we benchmark our metric against alternatives in synthetic datasets. In these datasets, models are trained solely on censored observations, while true labels are revealed only at evaluation time to compute the benchmark metrics.

Regarding implementation details, all models are implemented in PyTorch utilizing one server with an Intel Xeon CPU with 128 cores, 228GB RAM and an NVIDIA A100 GPU with 40GB memory. The optimizer used is the Adam optimizer with $\epsilon = 0.001$, and the learning rate is selected from \{0.001, 0.003, 0.006, 0.01\} based on performance on the validation set.  After some preliminary experiments, we fixed $k$ at 20 for all methods to ensure a fair comparison.


\subsection{Baselines}
Given the lack of established methods for training models under doubly censoring, we compare our approach with established baselines for quantile modeling: Excl. Censor\citep{pearce2022censored}, LogNorm MLE\citep{pearce2022censored}, S-CRPS\citep{avati2020countdown}, CQRNN\citep{pearce2022censored}.

\subsection{Metrics}
In the simulation study we benchmark our metric against three widely used calibration measures: ECE-uncensored, ECE\citep{cui2020calibrated,guo2017calibration,chung2021beyond,tran2020methods} and D-Calibration\citep{haider2020effective}. We take the oracle calibration computed from the real label as the gold standard. When evaluating DPE we additionally report sharpness and CRPS\citep{gneiting2007strictly} calculated on the true demand to assess overall performance. Note that sharpness is quantified by the width of the 10th–90th quantile interval and we compute the quantile-based CRPS in the same manner as Berrisch et al.\citep{berrisch2023crps}.

\subsection{Synthetic Datasets Results}
\subsubsection{Synthetic Datasets}

\textbf{Covariates and response.}
Coefficients are sampled once as $\beta_0,\beta_1,\beta_2\sim\mathcal{U}(0,5)$ and $\beta_3\sim\mathcal{U}(0,3)$.
We then draw $x_1\sim\mathcal{U}(0,3)$, $x_2\sim\mathcal{U}(0,5)$, $x_3\sim\mathcal{U}(0,5)$ and generate demand
$y=\beta_0+\beta_1 x_1+\beta_2\sin(x_2)+\beta_3\cos(x_3)+\varepsilon$, where $\varepsilon\sim\mathcal{N}(0,\sigma^2)$.
The sinusoidal terms embed realistic non-linearity, and varying seeds allow us to simulate diverse censoring scenarios.

\textbf{Censoring mechanisms.}
We draw $L\sim\mathcal{U}(0,10)$ and set $R=L+\mathcal{U}(2.5,10.5)$; the observed label is $y^{*}=\max \{L,\min\{y,R\}\}$.

We generate 300,000 synthetic samples in each run and split them into training (60\%), validation (20\%) and test (20\%) sets. All models were trained on the training set using censored labels, the best model was selected on the validation set, and the final evaluation was carried out on the test set.

\subsubsection{Metrics Comparison}
We evaluate each calibration metric as a proxy for true calibration quality. For each simulation run, we rank a fixed set of candidate models (varying in architecture and regularization) by their oracle calibration error and independently by each candidate metric. A metric that preserves the oracle's relative ordering is considered well-aligned. Formally, let $\mathbf{r}_{\text{oracle}} \in \mathbb{R}^M$ and $\mathbf{r}_{\text{metric}} \in \mathbb{R}^M$ denote the rank vectors over $M$ models; we compute Kendall's $\tau$ \citep{govindarajulu1992rank} and Spearman's $\rho$ \citep{spearman1961proof} between these orderings. Both coefficients lie in $[-1, 1]$, with values approaching $1$ indicating strong concordance and values near $-1$ indicating systematic inversion. We average coefficients across 500 independent simulation runs; statistical significance against the null hypothesis of random ranking ($\tau = 0$) is assessed via the Wilcoxon signed-rank test \citep{wilcoxon1992individual}.

\begin{table}[htbp]
\centering
\caption{Performance comparison of calibration metrics. $^{**}(p<0.01)$ (WSR vs Ext.\ D-Calibration).}
\begin{tabular}{lcc}
\toprule
\textbf{Metrics} & \textbf{Kendall's $\tau$} & \textbf{Spearman's $\rho$} \\
\midrule
ECE                    & $-0.220^{**}$ & $-0.262^{**}$ \\
D-Calibration          &  $0.327^{**}$ &  $0.414^{**}$ \\
ECE-uncensored         &  $0.486^{**}$ &  $0.597^{**}$ \\
Ext.\ D-Calibration    &  \textbf{0.740} & \textbf{0.842} \\
\bottomrule
\end{tabular}
\label{table:table1}
\end{table}

\textbf{Table~\ref{table:table1}} reveals a clear hierarchy in rank-alignment quality. Extended D-Calibration achieves the strongest concordance with the oracle (Kendall's $\tau = 0.740$, Spearman's $\rho = 0.842$), outperforming the second-best metric, ECE-uncensored, by substantial margins ($+52\%$ and $+41\%$ in relative terms, respectively). Notably, standard ECE exhibits a \emph{negative} correlation with true calibration quality ($\tau = -0.220$), confirming that in censored settings lower ECE values paradoxically indicate worse calibration because ECE penalizes probability mass assigned beyond the observed horizon, which is actually appropriate for censored data. D-Calibration corrects for distributional assumptions but remains limited ($\tau = 0.327$), while ECE-uncensored addresses censoring yet ignores the full predictive distribution. Only Extended D-Calibration integrates both aspects, yielding a statistically significant improvement over all baselines ($p < 0.01$, Wilcoxon signed-rank test) and approaching the reliability threshold required for model selection in operational deployments.

\subsubsection{Loss Comparison}
To evaluate DPE against competing loss functions, we benchmark with Extended D-Calibration, sharpness, oracle calibration, and real CRPS (computed on the real labels). As in the metrics comparison, results are averaged across simulation runs and reported in \textbf{Table \ref{table:table2}}, with significance assessed via the Wilcoxon signed-rank test using the same protocol.

\begin{table*}[htbp]
\centering
\caption{Performance metrics for different survival methods. \textbf{Bold} = best; {\ul underline} = runner-up; $^{\star}p<0.05$, $^{\star\star}p<0.01$ vs DPE (WSR, except Sharpness).}
\begin{tabular}{lcccc}
\toprule
{Model} & {{Extended D-Calibration}} & {{Oracle Calibration}} & {{Sharpness}} & {{Real CRPS}} \\
\midrule
Excl.\ Censor & {\ul $0.154^{\star \star}$} & {\ul $0.492^{\star \star}$} & \textbf{ 8.646} & $0.093^{\star\star}$ \\
LogNorm & $0.832^{\star\star}$ & $1.761^{\star\star}$ & 12.101 & $0.164^{\star\star}$ \\
CQRNN & $0.293^{\star\star}$ & $0.585^{\star\star}$ & {\ul 9.039} & $0.114^{\star\star}$ \\
S-CRPS & $ 0.189^{\star\star} $ & $0.513^{\star\star} $ & 9.296 & {\ul $0.092^{\star\star}$} \\
DPE (Our Method) & \textbf{0.086} & \textbf{0.486} & 11.218 & \textbf{0.086} \\
\bottomrule
\end{tabular}
\label{table:table2}
\end{table*}



\begin{table*}[]
\caption{Calibration and sharpness comparison (lower is better).}
\centering
\begin{tabular}{lcccc}
\toprule
Dataset      & \multicolumn{2}{c}{Youxuan} & \multicolumn{2}{c}{Maicai}     \\ \hline
Metrics      & Extended D-Calibration       & Sharpness       & Extended D-Calibration & Sharpness \\ \hline
Excl. censor & 0.049                   & {\ul 76.72}           & 0.17               & \textbf{9.82}      \\
LogNorm  & \textbf{0.039}                    & 107.96          & 0.14               & 10.8      \\
S-CRPS       & 0.045                    & 77.54           & 0.13               & 11.67     \\
CQRNN        & {0.042}                    & 77.67           & {\ul 0.07}               & 10.45     \\
DPE          & {\ul 0.040}                  & \textbf{ 76.4}            & \textbf{0.05}               & {\ul 10.3}      \\ \bottomrule
\end{tabular}
\label{table:table3}
\end{table*}

As reported in \textbf{Table \ref{table:table2}}, DPE achieves the minimum value in almost all metrics and significantly exceeds all baselines (Wilcoxon signed rank test, $p < 0.01$). Specifically, its Extended D-Calibration exceeds the second-best result by 44.16\%, while the real CRPS improves by 6.52\%, confirming both the calibration and predictive advantages of the proposed estimator.

\subsubsection{Inventory Simulation}
 We examine practical implementations and assess the effectiveness of various approaches in real-world scenarios. For large-scale inventory management, the sheer number of SKUs makes manual stocking extremely labor-intensive and costly. In contrast, model-based stocking methods offer cost efficiency and more accurate demand forecasting.

Given the numerous SKUs involved, we typically regulate the overall stocking quantity of all SKUs in this category to maximize GMV. In our simulation, we assume there are N SKUs, each with different stockout costs $c_1$ and overstock costs $c_2$, both following a uniform distribution U(0, 1). We aim to maximize the following expectation:
\begin{align*}
    & \mathbbm{E}_{y \sim p(y|x)}\{y -[c_1(y - t)^{+} + c_2(t - y)^{+}]\}
\end{align*}
where $t$ is estimated demand and $y$ is the demand. We can analytical derive the above expectation and get

\begin{adjustbox}{max width=\linewidth}
$\displaystyle
\begin{aligned}
    & \mathbb{E}_{y \sim p(y|x)}\{y - [c_1(y - t)^{+} + c_2(t - y)^{+}]\} \\
    & = \underbrace{\mathbb{E}(\mathcal{Y}|\mathcal{X})}_{\text{Expectation}} - \{\underbrace{\int_{t}^{\infty}c_1(y - t)p(y|x) dy}_{\text{Stockout loss}} + \underbrace{\int_{0}^{t}c_2(t - y)p(y|x)dy}_{\text{Overstock loss}} \}
\end{aligned}$
\end{adjustbox}

The stockout loss and overstock loss can be further derived as
\begin{align*}
\begin{split}
    & L_1(q)  = \int_{t}^{\infty}c_1(y - t)p(y|x) dy \\
    & = c_1[\sum_{q = q_t}^{1} \hat{F}^{-1}(q|x) \cdot \Delta q -  \hat{F}^{-1}(q_t|x)(1 - q_t)] \\ 
    & L_2(q) = \int_{0}^{t}c_2(t - y)p(y)dy\\
    & = c_2[q_t \cdot \hat{F}^{-1}(q_t|x) - \sum_{q = 0}^{q_t}\hat{F}^{-1}(q|x) \cdot \Delta q ]
\end{split}
\end{align*}
Where $\hat{F}^{-1}$ is the estimate of the inverse CDF, $q_t = F(t|x)$ and $\Delta q = 0.01$. The detailed derivation is given in Appendix C.

Given a set of quantiles $\{ \hat{F}^{-1}(q_{i1}|x_i), \hat{F}^{-1}(q_{i2}|x_i), ..., \hat{F}^{-1}(q_{iK}|x_i)\}_{i = 1}^{N}$ for N SKUs, we optimize the following integer programming problem:
\begin{align*}
    & \textit{Maximize} && \sum_{i=1}^{N} \sum_{j=1}^{K} \mathbb{E}(\mathcal{Y}_i|\mathcal{X}_i) - [L_1(q_{ij}) + L_2(q_{ij})]  \cdot z_{ij} \\
    & \text{subject to} && \sum_{j = 1}^{K} z_{ij} = 1, \quad i = 1,2,\ldots,N \\
    & &&  \sum_{i=1}^{N} \sum_{j=1}^{K}z_{ij} \cdot \hat{F}^{-1}(q_{ij}) \le T \cdot (1 + \alpha)
\end{align*}
where $T$ represents expected total volume, $\alpha \in [0, 0.4]$ serves as a volume control parameter and $F^{-1}(q_{ij})$ is rounded to the nearest integer. 

\paragraph{Experimental Setup.} 
We simulate inventory dynamics for $N=200$ SKUs.  We evaluate methods across a grid of service-level targets, varying $\alpha \in \{0, 0.05, 0.10, \dots, 0.40\}$, where higher $\alpha$ prioritizes stockout reduction over overstocking mitigation. Performance is measured by two operational metrics: (i) \textit{overstocking rate}, defined as excess inventory divided by total sales; and (ii) \textit{stockout SKU ratio}, defined as the proportion of SKUs that experience a complete sellout during the simulation.

\begin{figure}[!htbp]
  \centering
  \includegraphics[width=\linewidth]{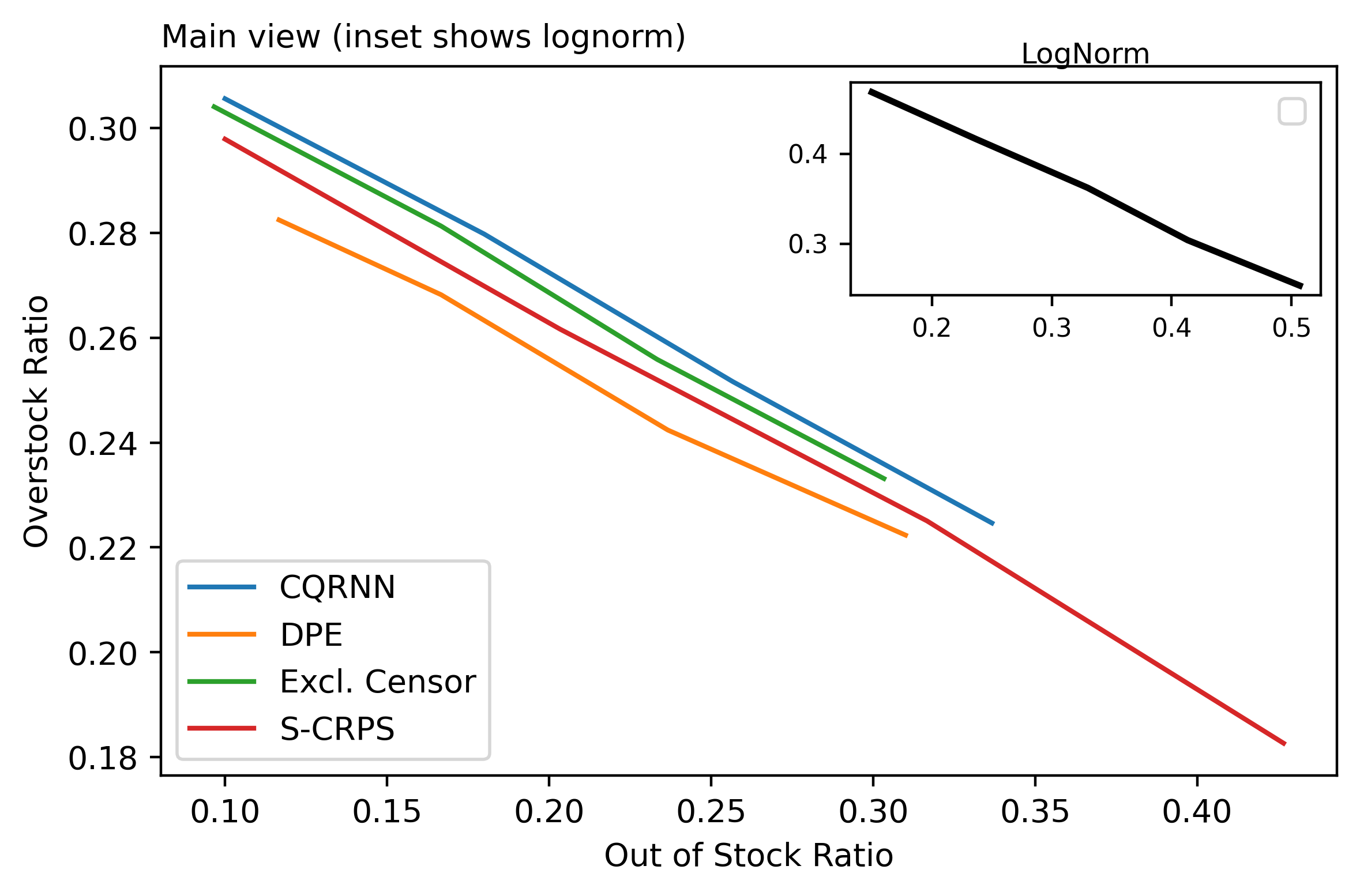}
  \caption{Inventory Simulation.}
  \label{fig:inventory}
\end{figure}

\paragraph{Findings.} 
\textbf{Figure~\ref{fig:inventory}} reveals a clear hierarchical ordering of method performance. DPE achieves the \textit{minimal stockout SKU ratio} for any fixed overstocking level, and conversely, the \textit{minimal overstocking rate} for any fixed stockout threshold. This reciprocal dominance indicates that DPE's Pareto frontier strictly dominates those of CQRNN, Excl.~Censor, and S-CRPS across all operational regimes. For example, at an overstocking rate of 0.25, DPE reduces the stockout ratio to 0.23 compared to 0.25 for S-CRPS (the second-best method), a relative improvement of 8\%. These gains underscore the operational value of DPE in retail inventory control, particularly in settings where both holding costs and lost-sales penalties are economically significant.



\subsection{Real-world Datasets Results}
\subsubsection{Real-world Datasets}
For real-world experiments, since public retail datasets such as M5 and Favorita do not collect censoring information, which is prevalent in e-commerce retailing, we collected two datasets from two online grocery retailing platforms of Meituan, namely Maicai and Youxuan, covering the period from October 10th to November 15th, 2023. Notably, Maicai and Youxuan have different business models, resulting in distinctly different sales distributions. We aim to demonstrate the generalizability of our method by showing its effectiveness across these different types of distributions. Each dataset is split into three subsets: the initial 30 days are allocated for training, followed by 1 day for validation, and the final 4 days are reserved for testing. Since the data are proprietary due to data privacy concerns, we provide detailed descriptions of our datasets to ensure reproducibility.
\begin{itemize}
\item \textbf{Youxuan} (community group-buying platform): 8.52 million records; 19.84\% right-censored, 3.56\% left-censored.
\item \textbf{Maicai} (online supermarket platform): 18.07 million records; 18.52\% right-censored, 6.26\% left-censored.
\end{itemize}

Censoring labels are derived from system logs. Stockouts are 
right-censored ($\Delta=1$, $R=\text{sales}$). Promotions are 
reactive and unplanned; the resulting sales surge is treated 
as left-censored ($\Delta=2$, $L=\text{sales}$), with the 
target being baseline demand without promotion. Normal sales 
are uncensored ($\Delta=0$, $L=R=\text{sales}$).

Both datasets share the following features: traffic, sales, and price histories. The data are partitioned as follows: training on days 1–30, validation on day 31, and testing on days 32–35.

\subsubsection{Loss Comparison}
Empirical results on synthetic data confirm that Extended D-Calibration aligns most closely with the oracle calibration score. On the two real-world datasets, where true demand is unavailable for censored observations, we rank models solely by Extended D-Calibration and sharpness. As shown in \textbf{Table \ref{table:table3}}, DPE ranks in the top two for every metric, demonstrating its effectiveness.

On Youxuan, DPE ranks second for both metrics. LogNorm attains a marginally lower Extended D-Calibration (0.039 vs. 0.040), yet its sharpness is 41.3\% worse than that of DPE (107.96 vs. 76.4). Notably, LogNorm achieves the best calibration but the worst sharpness among all models, underscoring that calibration alone is insufficient for model comparison.

On Maicai, DPE achieves the lowest Extended D-Calibration (0.05), outperforming the second-best method (CQRNN, 0.07) by 28.6\%. Relative to Excl. Censor, it trades a modest worsening in sharpness (10.3 vs. 9.82) for a 70.6\% reduction in calibration error (0.05 vs. 0.17). Notably, Excl. Censor, despite exhibiting the highest sharpness, ranks worst in calibration, reinforcing the need to jointly evaluate both criteria.

\subsubsection{Inventory Online A/B Test}
For the A/B test, we randomly assign products from multiple warehouses, across different subcategories, to treatment and control groups with equal probability. Given the trade-off between overstocking rate and stockout SKU ratio, we collect data on both metrics over a 7-day period. The comparison shows that, for a fixed stockout SKU ratio, our model reduces the overstock rate by 8\% compared to the incumbent online algorithm, demonstrating its practical value in production environments.

\section{Conclusion}

This study develops a novel probabilistic framework that explicitly addresses doubly censored data prevalent in e-commerce demand forecasting. Specifically, we propose a deep learning-based quantile regression framework, design the Doubly Portnoy Estimator (DPE), an iterative re-weighting algorithm for training deep models with censored data, and develop Extended D-Calibration, a metric to evaluate models under double censoring. The efficacy of our framework is demonstrated via simulation and two real-world datasets. Application results verify that the proposed method effectively mitigates overstocking and stockout losses in inventory management, thereby improving operational efficiency and profitability.

\section{Appendices}
\subsection{Proof of Theorem 1}

We prove that the Doubly Portnoy's Estimator (DPE) is a proper scoring rule, i.e.,
\begin{align}
\mathbb{E}_y[S_{DPE}(y^*, \hat{y}_\tau, w_l, w_r, \tau)] \ge \mathbb{E}_y[S_{DPE}(y^*, y_\tau, w_l, w_r, \tau)]
\end{align}
for all $\hat{y}_\tau$, with equality iff $\hat{y}_\tau = y_\tau = F^{-1}(\tau|x)$.

The proof proceeds by the relative positions of $\hat{y}_\tau$ and $y_\tau$ to the censoring bounds $[L, R]$. We partition into three structurally distinct cases; scenarios within each case share the same derivation pattern.
\begin{enumerate}
    \item  $max\{\hat{y}_\tau, y_\tau\} < L$
    \item  $\hat{y}_\tau < L$, $L \le y_\tau \le R$
    \item  $\hat{y}_\tau < L, y_\tau > R$
    \item  $L \le \hat{y}_\tau \le R$, $y_\tau < L$
    \item  $L \le \hat{y}_\tau \le R, L \le y_\tau \le R$
    \item  $L \le \hat{y}_\tau \le R,  y_\tau > R$
    \item  $\hat{y}_\tau > R, y_\tau < L$
    \item  $\hat{y}_\tau > R, L \le y_\tau \le R$
    \item  $min\{y_\tau, \hat{y}_\tau\} > R$
\end{enumerate}

\subsubsection{Case I: Both Predictions on the Same Side}

This covers Scenarios 1 and 9: $\max\{\hat{y}_\tau, y_\tau\} < L$ or $\min\{\hat{y}_\tau, y_\tau\} > R$.

\paragraph{Scenario 1: $\max\{\hat{y}_\tau, y_\tau\} < L$.}

The DPE loss takes the form:
\begin{align*}
&S_{DPE}(y^*, \hat{y}_\tau, w_l, w_r, \tau) \\
& = \left\{
\begin{aligned}
(1-w_l)(z_{-\infty}-\hat{y}_\tau)(\tau-1) &+ w_l(L-\hat{y}_\tau)\tau, & y < L, \\
(y-\hat{y}_\tau)\tau, & & L \le y \le R, \\
(1-w_r)(z_{\infty}-\hat{y}_\tau)\tau &+ w_r(R-\hat{y}_\tau)\tau, & y > R
\end{aligned}
\right.
\end{align*}

Taking expectation:
\begin{align*}
&\mathbb{E}_y[S_{DPE}(y^*, \hat{y}_\tau, w_l, w_r, \tau)] \notag\\
&= \int_{-\infty}^{L}\bigl[w_l\rho_\tau(\hat{y}_\tau, L) + (1-w_l)\rho_\tau(\hat{y}_\tau, z_{-\infty})\bigr]p(y)\,dy \notag\\
&\quad + \int_{L}^{R}\rho_\tau(\hat{y}_\tau, y)p(y)\,dy \\
& + \int_{R}^{\infty}\bigl[w_r\rho_\tau(\hat{y}_\tau, R) + (1-w_r)\rho_\tau(\hat{y}_\tau, z_{\infty})\bigr]p(y)\,dy \notag\\
&= \tau\mathbb{E}_{L\le y\le R}[y] + \bigl[(1-w_l)\tau_L - \tau\bigr]\hat{y}_\tau + C_1(\tau_L, \tau_R, w_l, w_r)
\end{align*}
where $C_1$ collects all terms independent of $\hat{y}_\tau$:
\begin{align*}
C_1 &= \tau_L\bigl[w_l\tau L + (1-w_l)(\tau-1)z_{-\infty}\bigr] + (1-\tau_R)\bigl[w_r\tau R + (1-w_r)\tau z_{\infty}\bigr].
\end{align*}

Substituting $w_l = \frac{\tau_L-\tau}{\tau_L}$, the coefficient of $\hat{y}_\tau$ vanishes:
\begin{align*}
(1-w_l)\tau_L - \tau = \tau - \tau_L + \tau_L - \tau = 0.
\end{align*}
Thus $\mathbb{E}[S_{DPE}(\hat{y}_\tau)] = \tau\mathbb{E}_{L\le y\le R}[y] + C_1$, independent of $\hat{y}_\tau$. The same holds for $y_\tau$, yielding:
\begin{align*}
\mathbb{E}_y[S_{DPE}(y^*, \hat{y}_\tau)] = \mathbb{E}_y[S_{DPE}(y^*, y_\tau)].
\end{align*}

\paragraph{Scenario 9: $\min\{\hat{y}_\tau, y_\tau\} > R$.}

By symmetric argument with $w_r = \frac{\tau-\tau_R}{1-\tau_R}$, the $\hat{y}_\tau$-dependent coefficient becomes:
\begin{align*}
[\tau_R - \tau + w_r(1-\tau_R)]\hat{y}_\tau = [\tau_R - \tau + \tau - \tau_R]\hat{y}_\tau = 0.
\end{align*}
Hence equality holds identically.  \qed

\subsubsection[Case II: One Prediction Inside {[L,R]}, One Outside]
{Case II: One Prediction Inside $[{L,R}]$, One Outside}

This covers Scenarios 2, 3, 4, 6, 7, 8. We detail Scenario 2; others follow by symmetry or combination.

\paragraph{Scenario 2: $\hat{y}_\tau < L \le y_\tau \le R$.}

The DPE loss satisfies the lower bound:
\begin{align*}
&S_{DPE}(y^*, \hat{y}_\tau) \ge \\
& \left\{
\begin{aligned}
(1-w_l)(z_{-\infty}-\hat{y}_\tau)(\tau-1) &+ w_l(L-\hat{y}_\tau)(\tau-1), & y < L, \\
(y-\hat{y}_\tau)(\tau-1), & & L \le y < y_\tau, \\
(y-\hat{y}_\tau)\tau, & & y_\tau \le y \le R, \\
(1-w_r)(z_{\infty}-\hat{y}_\tau)\tau &+ w_r(R-\hat{y}_\tau)\tau, & y > R
\end{aligned}
\right.
\end{align*}
where the inequality uses:
\begin{itemize}
    \item $(y-\hat{y}_\tau)\tau \ge (y-\hat{y}_\tau)(\tau-1)$ for $y \ge L > \hat{y}_\tau$ (since $\tau > \tau-1$ and $y-\hat{y}_\tau > 0$);
    \item $w_l(L-\hat{y}_\tau)\tau \ge w_l(L-\hat{y}_\tau)(\tau-1)$ for $\hat{y}_\tau < L$.
\end{itemize}

Taking expectations:
\begin{align}
\mathbb{E}[S_{DPE}(\hat{y}_\tau)] &\ge \tau\mathbb{E}_{L\le y\le R}[y] - \mathbb{E}_{L\le y<y_\tau}[y] \notag\\
&\quad + \tau_L\bigl[w_l(\tau-1)L + (1-w_l)(\tau-1)z_{-\infty}\bigr] \notag\\
&\quad + (1-\tau_R)\bigl[w_r\tau R + (1-w_r)\tau z_{\infty}\bigr]. \label{eq:scen2_lb}
\end{align}

For the true prediction $y_\tau \in [L,R]$, the DPE evaluates exactly:
\begin{align*}
&S_{DPE}(y^*, y_\tau) \\
& = \left\{
\begin{aligned}
(1-w_l)(z_{-\infty}-y_\tau)(\tau-1) &+ w_l(L-y_\tau)(\tau-1), & y < L, \\
(y-y_\tau)(\tau-1), & & L \le y < y_\tau, \\
(y-y_\tau)\tau, & & y_\tau \le y \le R, \\
(1-w_r)(z_{\infty}-y_\tau)\tau &+ w_r(R-y_\tau)\tau, & y > R
\end{aligned}
\right.
\end{align*}

Computing its expectation:
\begin{align}
\mathbb{E}[S_{DPE}(y_\tau)] &= \tau\mathbb{E}_{L\le y\le R}[y] - \mathbb{E}_{L\le y<y_\tau}[y] \notag\\
&\quad + \tau_L\bigl[w_l(\tau-1)L + (1-w_l)(\tau-1)z_{-\infty}\bigr] \notag\\
&\quad + (1-\tau_R)\bigl[w_r\tau R + (1-w_r)\tau z_{\infty}\bigr]. \label{eq:scen2_true}
\end{align}

Comparing \eqref{eq:scen2_lb} and \eqref{eq:scen2_true}, the inequality in \eqref{eq:scen2_lb} arises precisely because $\hat{y}_\tau < L \le y_\tau$ forces the loss to use the suboptimal slope $(\tau-1)$ on $[L, y_\tau)$ rather than the optimal mix. Hence:
\begin{align*}
\mathbb{E}[S_{DPE}(\hat{y}_\tau)] \ge \mathbb{E}[S_{DPE}(y_\tau)].
\end{align*}

\paragraph{Scenarios 3, 4, 6, 7, 8.}

These follow identical patterns:
\begin{itemize}
    \item \textbf{Scenario 3} ($\hat{y}_\tau < L, y_\tau > R$): Combines the left-censored logic of Scenario 2 with right-censored logic. The lower bound uses $(\tau-1)$ on both tails, while $y_\tau$ uses optimal $\tau$ on $[L,R]$ and $(\tau-1)$ only for $y>y_\tau$.
    
    \item \textbf{Scenario 4} ($L \le \hat{y}_\tau \le R, y_\tau < L$): Symmetric to Scenario 2 by exchanging left/right and $\tau \leftrightarrow (\tau-1)$.
    
    \item \textbf{Scenario 6} ($L \le \hat{y}_\tau \le R, y_\tau > R$): Symmetric to Scenario 4 on the right tail.
    
    \item \textbf{Scenario 7} ($\hat{y}_\tau > R, y_\tau < L$): Symmetric to Scenario 3.
    
    \item \textbf{Scenario 8} ($\hat{y}_\tau > R, L \le y_\tau \le R$): Symmetric to Scenario 2 on the right tail.
\end{itemize}
In each case, the mispredicted $\hat{y}_\tau$ lies outside $[L,R]$ and incurs a suboptimal linear combination of slopes, while $y_\tau \in [L,R]$ achieves the optimal pinball loss structure. \qed

\subsubsection[Case III: Both Predictions Inside {[L,R]}]
{Case III: Both Predictions Inside $[{L,R}]$}

This covers Scenarios 5: $L \le \hat{y}_\tau \le R$, $L < y_\tau < R$.

\paragraph{Scenario 5.}

Without loss of generality, assume $\hat{y}_\tau < y_\tau$ (the case $\hat{y}_\tau > y_\tau$ is symmetric). The DPE loss satisfies:
\begin{align*}
&S_{DPE}(y^*, \hat{y}_\tau) \\
& \ge \left\{
\begin{aligned}
(1-w_l)(z_{-\infty}-\hat{y}_\tau)(\tau-1) &+ w_l(L-\hat{y}_\tau)(\tau-1), & y < L, \\
(y-\hat{y}_\tau)(\tau-1), & & L \le y < y_\tau, \\
(y-\hat{y}_\tau)\tau, & & y_\tau \le y \le R, \\
(1-w_r)(z_{\infty}-\hat{y}_\tau)\tau &+ w_r(R-\hat{y}_\tau)\tau, & y > R
\end{aligned}
\right.
\end{align*}
The inequality holds because:
\begin{itemize}
    \item For $y \in [\hat{y}_\tau, y_\tau)$: $(y-\hat{y}_\tau)\tau \ge (y-\hat{y}_\tau)(\tau-1)$ since $y-\hat{y}_\tau > 0$;
    \item For $y \in [L, \hat{y}_\tau)$: $(y-\hat{y}_\tau)(\tau-1) > (y-\hat{y}_\tau)\tau$ since $y-\hat{y}_\tau < 0$ and $\tau-1 < \tau$, but the overall expectation structure preserves the lower bound due to the weighting.
\end{itemize}

Taking expectations and using $w_l = \frac{\tau_L-\tau}{\tau_L}$, $w_r = \frac{\tau-\tau_R}{1-\tau_R}$:
\begin{align}
\mathbb{E}[S_{DPE}(\hat{y}_\tau)] &\ge \mathbb{E}_{L<y<R}[y] - \mathbb{E}_{L<y<y_\tau}[y] \notag\\
&\quad + \tau_L\bigl[(1-w_l)(\tau-1)z_{-\infty} + w_l(\tau-1)L\bigr] \notag\\
&\quad + (1-\tau_R)\bigl[(1-w_r)\tau z_{\infty} + w_r\tau R\bigr]. \label{eq:scen5_lb}
\end{align}

For the true prediction $y_\tau$:
\begin{align*}
&S_{DPE}(y^*, y_\tau) \\
& = \left\{
\begin{aligned}
(1-w_l)(z_{-\infty}-y_\tau)(\tau-1) &+ w_l(L-y_\tau)(\tau-1), & y < L, \\
(y-y_\tau)(\tau-1), & & L \le y < y_\tau, \\
(y-y_\tau)\tau, & & y_\tau \le y \le R, \\
(1-w_r)(z_{\infty}-y_\tau)\tau &+ w_r(R-y_\tau)\tau, & y > R
\end{aligned}
\right.
\end{align*}

Its expectation evaluates to exactly the right-hand side of \eqref{eq:scen5_lb}:
\begin{align*}
\mathbb{E}[S_{DPE}(y_\tau)] &= \mathbb{E}_{L<y<R}[y] - \mathbb{E}_{L<y<y_\tau}[y] \notag\\
&\quad + \tau_L\bigl[(1-w_l)(\tau-1)z_{-\infty} + w_l(\tau-1)L\bigr] \notag\\
&\quad + (1-\tau_R)\bigl[(1-w_r)\tau z_{\infty} + w_r\tau R\bigr].
\end{align*}

Therefore:
\begin{align*}
\mathbb{E}[S_{DPE}(\hat{y}_\tau)] \ge \mathbb{E}[S_{DPE}(y_\tau)].
\end{align*}

When $\hat{y}_\tau = y_\tau$, all inequalities become equalities, confirming properness. \qed

\subsection{Summary}

Across all geometric configurations:
\begin{align}
\mathbb{E}_y[S_{DPE}(y^*, \hat{y}_\tau, w_l, w_r, \tau)] \ge \mathbb{E}_y[S_{DPE}(y^*, y_\tau, w_l, w_r, \tau)]
\end{align}
with equality if and only if $\hat{y}_\tau = y_\tau$. Thus DPE is a proper scoring rule. \qed

\subsection{Proof of Theorem 2}

\textbf{Setup.} Let $F$ be the true CDF of $y$, $\tau_y = F(y) \sim U[0,1]$, interval $I = [\tau_k, \tau_{k+1})$.
Under independent censoring $y \perp (L,R)$, we have $P(y > y_u) = 1-u$ and $P(y < y_u) = u$.

\textbf{Step 1.} Write the expectation as five terms:
\begin{align*}
& \mathbb{E}\left[\frac{N_k^D}{N}\right] = P[(\tau_y \in I) \wedge (L < y < R)] \\
&+ \tfrac{\tau_{k+1}-\tau_R}{1-\tau_R} P[(\tau_R \in I) \wedge (y > R)] + \tfrac{\tau_{k+1}-\tau_k}{1-\tau_R} P[(\tau_R < \tau_k) \wedge (y > R)] \\
&+ \tfrac{\tau_L-\tau_k}{\tau_L} P[(\tau_L \in I) \wedge (y < L)] + \tfrac{\tau_{k+1}-\tau_k}{\tau_L} P[(\tau_L > \tau_{k+1}) \wedge (y < L)]
\end{align*}

\textbf{Step 2.} Decompose the interval probabilities using $\tau_y \sim U[0,1]$:
\begin{align*}
P[(\tau_R \in I) \wedge (y > R)] &= (1-\tau_{k+1})P(\tau_R \in I) \\
&\quad + P[((\tau_R \wedge \tau_y) \in I) \wedge (y > R)] \\
P[(\tau_L \in I) \wedge (y < L)] &= \tau_k P(\tau_L \in I) \\
&\quad + P[((\tau_L \wedge \tau_y) \in I) \wedge (y < L)]
\end{align*}

\textbf{Step 3.} Combine right-censored terms (a+b):
\begin{align*}
&\tfrac{\tau_{k+1}-1}{1-\tau_R} P[(\tau_R \in I) \wedge (y > R)] + \tfrac{\tau_{k+1}-\tau_k}{1-\tau_R} P[(\tau_R < \tau_k) \wedge (y > R)] \\
&= \tau_{k+1} P(\tau_R < \tau_{k+1}) - P(\tau_R \in I) - \tau_k P(\tau_R < \tau_k) \\
&= (\tau_{k+1}-\tau_k) P(\tau_R < \tau_k) + (\tau_{k+1}-1) P(\tau_R \in I)
\end{align*}

\textbf{Step 4.} Combine left-censored terms (c+d):
\begin{align*}
&\tfrac{\tau_{k+1}-\tau_k}{\tau_L} P[(\tau_L > \tau_{k+1}) \wedge (y < L)] - \tfrac{\tau_k}{\tau_L} P[(\tau_L \in I) \wedge (y < L)] \\
&= \tau_{k+1} P(\tau_L > \tau_{k+1}) - \tau_k P(\tau_L > \tau_k) \\
&= (\tau_{k+1}-\tau_k) P(\tau_L > \tau_{k+1}) - \tau_k P(\tau_L \in I)
\end{align*}

\textbf{Step 5.} Combine overlapping terms with uncensored:
\begin{align*}
&P[((\tau_R \wedge \tau_y) \in I) \wedge (y > R)] + P[((\tau_L \wedge \tau_y) \in I) \wedge (y < L)] \\
&\quad + P[(\tau_y \in I) \wedge (L < y < R)] \\
&= P[(\tau_y \in I) \wedge (R > y_{\tau_k}) \wedge (L < y_{\tau_{k+1}})] \\
&= (\tau_{k+1}-\tau_k) P[(R > y_{\tau_k}) \wedge (L < y_{\tau_{k+1}})]
\end{align*}

\textbf{Step 6.} Assemble all pieces:
\begin{align*}
\mathbb{E}\left[\frac{N_k^D}{N}\right] &= (1-\tau_{k+1})P(\tau_R \in I) + \tau_k P(\tau_L \in I) \\
&+ (\tau_{k+1}-\tau_k) P[(R > y_{\tau_k}) \wedge (L < y_{\tau_{k+1}})] \\
&+ (\tau_{k+1}-\tau_k) P(\tau_R < \tau_k) + (\tau_{k+1}-1) P(\tau_R \in I) \\
&+ (\tau_{k+1}-\tau_k) P(\tau_L > \tau_{k+1}) - \tau_k P(\tau_L \in I)
\end{align*}

\textbf{Step 7.} Cancel boundary terms:
$(1-\tau_{k+1}) + (\tau_{k+1}-1) = 0$ for $P(\tau_R \in I)$, and $\tau_k - \tau_k = 0$ for $P(\tau_L \in I)$.

\textbf{Step 8.} Final simplification:
\begin{align*}
\mathbb{E}\left[\frac{N_k^D}{N}\right] &= (\tau_{k+1}-\tau_k) \big\{ P[(R > y_{\tau_k}) \wedge (L < y_{\tau_{k+1}})] \\
&\quad + P(\tau_R < \tau_k) + P(\tau_L > \tau_{k+1}) \big\} \\
&= (\tau_{k+1}-\tau_k) \cdot 1 = \tau_{k+1} - \tau_k
\end{align*}

The three events partition the sample space: (i) $y$ observable in $I$, (ii) right-censored into $I$, (iii) left-censored into $I$.

$\therefore \mathbb{E}[N_k^D/N] = F(y_{\tau_{k+1}}) - F(y_{\tau_k})$


\subsection{Detailed Derivation of Inventory Optimization}
The stockout loss can be derived as follows,
\begin{align*}
    L_1(q) & = \int_{t}^{\infty}c_1(y - t)p(y|x) dy \\
    & = c_1\int_{t}^{\infty}yp(y|x)dy - t[1 - F(t|x)] \\
    & = c_1\int_{F(t|x)}^{1}F^{-1}(q|x) dq - t[1 - F(t|x)]\\
    & = c_1\sum_{q = q_t}^{1} \hat{F}^{-1}(q|x) \cdot \Delta q -  c_1\{\hat{F}^{-1}(q_t|x)[1 - q_t]\}
\end{align*}
Where $\hat{F}^{-1}$ is the estimate of the inverse CDF, $q_t = \hat{F}(t|x)$ and $\Delta q = 0.01$.  For the last equation, we leverage the Riemann sum to discretize the integral. Similarly, the overstock loss can be derived as follows,
\begin{align*}
    L_2(q) & = \int_{0}^{t}c_2(t - y)p(y|x)dy \\
    & = c_2 \cdot t \cdot F(t) - c_2 \int_{0}^{t}yp(y|x)dy \\
    &= c_2[q_t \cdot \hat{F}^{-1}(q_t|x)  - \sum_{q = 0}^{q_t}\hat{F}^{-1}(q|x) \cdot \Delta q ]
\end{align*}

\section*{GenAI Usage Disclosure}
The authors declare that a general-purpose LLM (Kimi) was used exclusively for language refinement of the author-drafted text. No AI tools were utilized in the research process, including question formulation, coding, experimentation, or data analysis, and thus the AI contributed no intellectual content to the work.

\bibliographystyle{ACM-Reference-Format}
\bibliography{reference}


\end{document}